\documentclass{article}
\usepackage[T1]{fontenc}
\usepackage{arxiv}
\usepackage{wrapfig}
\usepackage{xcolor}
\usepackage{hyperref}
\usepackage{microtype}
\usepackage{graphicx}
\usepackage{amsmath,amssymb}
\usepackage{algorithm}
\usepackage[noend]{algpseudocode}
\usepackage{tablefootnote}
\providecommand{\Description}[1]{}
\usepackage{url}
\usepackage{booktabs}
\usepackage{amsfonts}
\usepackage{amsthm}
\usepackage{nicefrac}
\usepackage{multirow}
\usepackage[numbers]{natbib}
\setcitestyle{numbers,square}
\usepackage{mathtools}
\usepackage{subcaption}

\definecolor{darkblue}{RGB}{0, 55, 120}

\definecolor{jadegreen}{RGB}{0, 102, 77}
\definecolor{carmine}{RGB}{150, 0, 24}
\definecolor{amber}{RGB}{204, 85, 0}
\definecolor{burntorange}{rgb}{0.8, 0.33, 0.0}
\definecolor{darkred}{rgb}{0.55, 0.0, 0.0}

\newtheorem{theorem}{Theorem}
\newtheorem{proposition}{Proposition}

\newtheorem{definition}{Definition}

\title{Decomposing Private Image Generation via Coarse-to-Fine Wavelet Modeling}

\author{
Jasmine Bayrooti \\
University of Cambridge\thanks{Work done at Google as part of student researcher engagement.} \\
Cambridge, United Kingdom \\
\texttt{jgb52@cam.ac.uk} \\
\And
Weiwei Kong \\
Google Research \\
New York, NY, USA \\
\texttt{weiweikong@google.com} \\
\And
Natalia Ponomareva \\
Google Research \\
New York, NY, USA \\
\texttt{nponomareva@google.com} \\
\And
Carlos Esteves \\
Google Research \\
New York, NY, USA \\
\texttt{machc@google.com} \\
\And
Ameesh Makadia \\
Google Research \\
New York, NY, USA \\
\texttt{makadia@google.com} \\
\And
Amanda Prorok \\
University of Cambridge \\
Cambridge, United Kingdom \\
\texttt{asp45@cam.ac.uk}
}

\begin{document}

\maketitle

\begin{abstract}
Generative models trained on sensitive image datasets risk memorizing and reproducing individual training examples, making strong privacy guarantees essential. While differential privacy (DP) provides a principled framework for such guarantees, standard DP finetuning (e.g., with DP-SGD) often results in severe degradation of image quality, particularly in high-frequency textures, due to the indiscriminate addition of noise across all model parameters. In this work, we propose a spectral DP framework based on the hypothesis that the most privacy-sensitive portions of an image are often low-frequency components in the wavelet space (e.g., facial features and object shapes)
while high-frequency components are largely generic and public. Based on this hypothesis, we propose the following two-stage framework for DP image generation with coarse image intermediaries: (1) DP finetune an autoregressive spectral image tokenizer model on the low-resolution wavelet coefficients of the sensitive images, and (2) perform high-resolution upsampling using a publicly pretrained super-resolution model. By restricting the privacy budget to the global structures of the image in the first stage, and leveraging the post-processing property of DP for detail refinement, we achieve promising trade-offs between privacy and utility. Experiments on the MS-COCO and MM-CelebA-HQ datasets show that our method generates images with improved quality and style capture relative to other leading DP image frameworks.
\end{abstract}

\section{Introduction}

Recent advances in deep generative modeling have enabled high-quality image synthesis from textual descriptions, supporting applications such as creative image generation \cite{gafni2022make, rombach2022high} and interactive editing~\cite{nichol2021glide, avrahami2022blended}. However, training or finetuning these models on sensitive image datasets, including medical images and private photo collections, raises data-privacy concerns as generative models have been shown to memorize and reproduce information about individual training examples~\cite{carlini2023extracting, buchanan2025edge}.

Differential privacy (DP) provides a rigorous mathematical framework to bound this risk by ensuring that model outputs are statistically insensitive to the inclusion or exclusion of any single training example. In practice, most DP methods enforce this guarantee during training or finetuning using variants of differentially private stochastic gradient descent (DP-SGD)~\cite{abadi2016deep}. While DP-SGD provides strong formal privacy guarantees, it introduces a privacy-utility trade-off where achieving strong privacy levels requires aggressive gradient clipping and noise injection, which can overwhelm the learning signal in high-dimensional generative models. This often leads to degradation in visual quality, particularly for complex image distributions with fine-grained structure, such as human faces. As a result, existing DP image generation methods struggle to produce high-fidelity samples on challenging benchmarks, even when leveraging large pretrained models or parameter-efficient finetuning strategies. 

\begin{figure}
    \centering
    \includegraphics[width=0.9\linewidth]{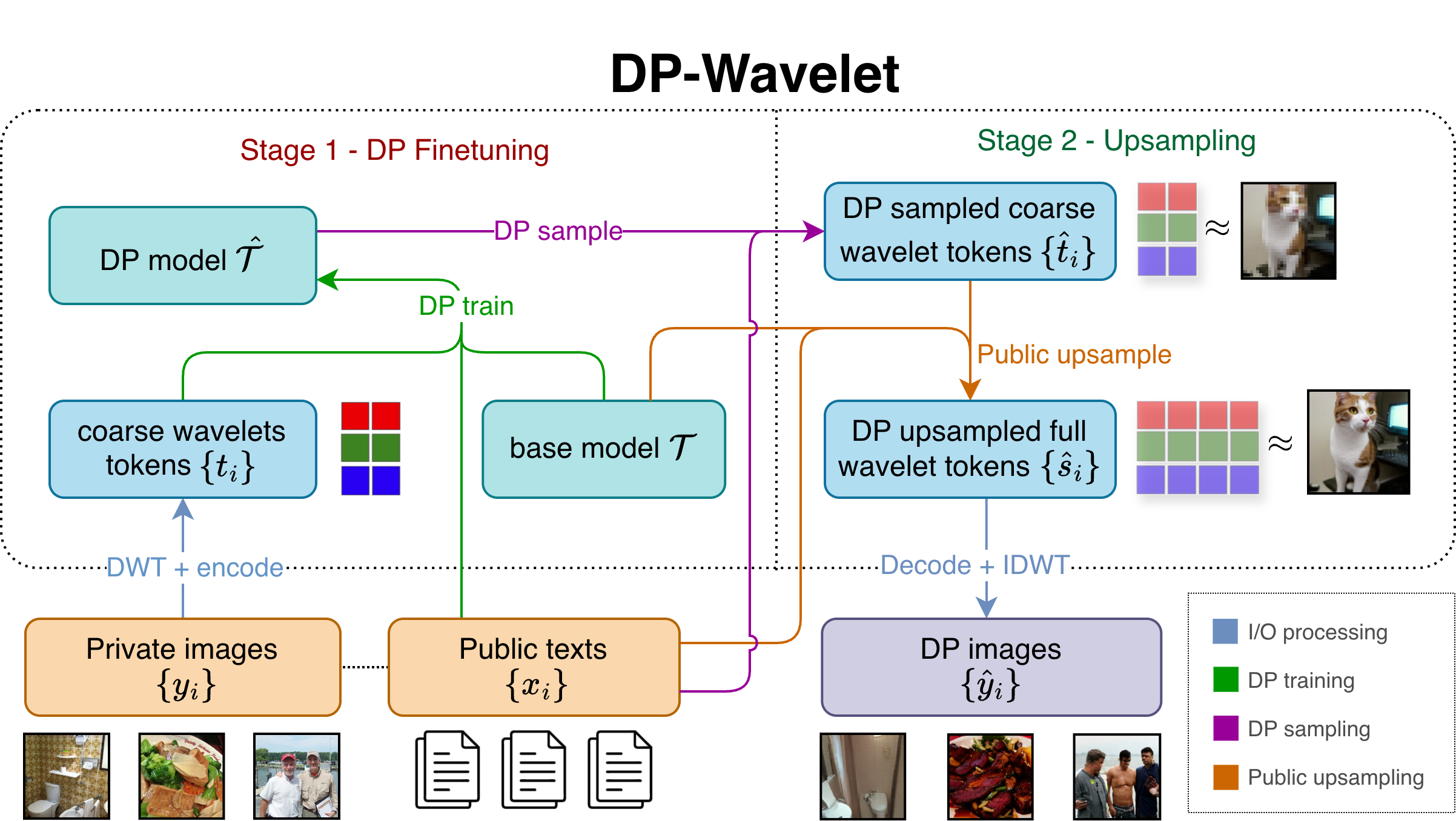}
    \caption{Illustration of the DP-Wavelet method.}
    \label{fig:dp_wavelet_overview}
    \Description{Overview of DP-Wavelet}
\end{figure}

To address these challenges, we introduce a DP image generation framework that improves the signal-to-noise properties of private training by carefully selecting which aspects of the image representation are learned under privacy constraints. Our approach is based on two hypotheses regarding the spectral structure of images and the interaction between DP-SGD and representation granularity. First, we hypothesize that the majority of the semantic signal that captures global structure, style, and content is concentrated in the low-frequency components of an image. This is supported by the energy compaction property of wavelet representations and sparsity of higher-frequency bands~\cite{mallat2002theory}. Second, we hypothesize that fine-scale image details such as skin textures and localized patterns are sufficiently generic to be synthesized by publicly pretrained models without direct access to private data. Together, these observations suggest that privacy budgets should be allocated to coarse, semantically meaningful image structure, while fine details can be generated as post-processing without additional privacy cost.

Building on these insights, we propose DP-Wavelet, a coarse-to-fine method for DP text-to-image generation using wavelet-based image intermediaries. In our setting, text prompts are assumed to be public, while images contain sensitive information. Given a dataset of public texts $x=\{x_i\}$ (e.g., captions and prompts) paired with private images $y=\{y_i\}$, our goal is to generate a set of DP synthetic images $\hat y = \{\hat y_i\}$ conditioned on $x$.

DP-Wavelet operates in two stages. In the first stage, we DP finetune a publicly pretrained wavelet-based image generation model $\mathcal{T}$ on public texts $x$ and a tokenization $\{t_i\}$ of the coarse, low-frequency wavelet components of the corresponding private images $y$, yielding the DP-finetuned model $\hat{\mathcal T}$. In the second stage, $\hat{\mathcal T}$ is used to generate a sequence of DP coarse wavelet tokens $\{\hat{t}_i\}$ conditioned on the texts $x$. The coarse tokens are then autoregressively completed into full wavelet representations $\{s_i\}$ using the frozen pretrained model $\mathcal T$ and finally decoded into images $\hat{y}$. An overview of this process is shown in Figure~\ref{fig:dp_wavelet_overview}, where the Discrete Wavelet Transform (DWT) and its inverse (IDWT) are used to move between pixel and wavelet frequency spaces.

\noindent\textbf{Motivating applications}. The public-text/private-image setting that we consider arises in scenarios where textual descriptions are non-sensitive but images contain confidential information, such as medical imaging datasets where clinical reports may be shareable while the underlying scans are protected. A second use case is style-preserving synthesis, e.g., modeling an artist’s style from public descriptions, while treating individual works as the private unit. DP-Wavelet also extends naturally to settings where texts are private by generating DP prompts using a separate DP text model and using these prompts for sampling; see Appendix~\ref{app:synth_motivation} for an extended discussion.\\

\noindent \textbf{Contributions}.
Our contributions are summarized as follows:
\begin{enumerate}
    \item We introduce DP-Wavelet, the first autoregressive DP method for text-to-image synthesis based on coarse image intermediaries. Unlike prior approaches that apply DP to dense latent or token representations, DP-Wavelet enforces privacy only on low-frequency wavelet components that capture global image structure.
    \item We state and operationalize the hypothesis that privacy budgets should be preferentially allocated to low-frequency image components. By restricting DP optimization to these subbands, our approach improves the signal-to-noise ratio under DP-SGD relative to pixel-space training.
    \item We empirically show that DP-Wavelet achieves competitive distributional quality (FID) and style consistency (LPIPS) relative to text-to-image baselines on the MS-COCO and MM-CelebA-HQ datasets. Notably, DP-Wavelet can produce DP coarse image intermediaries that reliably capture global structure, color composition, and stylistic cues of the underlying data.
\end{enumerate}

\subsection{Related Work}

In this section, we review prior work on non-private and private image synthesis as well as coarse-to-fine generation. \\

\noindent\textbf{Image synthesis}. Several foundational approaches like VAEs~\cite{kingma2013auto} and GANs~\cite{goodfellow2014generative} established the abilities of generative modeling, although they face trade-offs between sample sharpness and training stability. Modern state-of-the-art pipelines address this by decoupling representation learning from generation by first mapping images into compact latent or token spaces using learned autoencoders or tokenizers and then synthesizing images within this space. Diffusion-based approaches generate images by iteratively denoising latent representations across a sequence of noise levels~\cite{ho2020denoising, nichol2021improved, rombach2022high}. In contrast, autoregressive (AR) methods formulate image synthesis as a sequence prediction problem over discrete tokens and use transformer architectures to model token dependencies~\cite{esser2021taming, yu2022scaling}. While diffusion models tend to excel at high-frequency texture synthesis, AR models benefit from maximum-likelihood training and favorable scaling behavior~\cite{tian2024visual}, which we hypothesize may better preserve learning signal under differential privacy. Therefore, our work builds on this AR paradigm and leverages spectral tokenization to enable coarse-to-fine image generation. \\

\noindent\textbf{Private image synthesis}. Early research on DP image generation primarily focused on training GANs with DP-SGD~\cite{xie2018differentially, torkzadehmahani2019dp, bie2023private, chen2020gs, NEURIPS2021_67ed9474}. While effective on simple datasets like MNIST, these methods struggle to scale to more complex image datasets since DP noise injection exacerbates the instability of adversarial training. Therefore, more recent work applies DP to diffusion models, which are more robust to noisy optimization. Most DP diffusion methods rely on finetuning large public pretrained models to preserve semantic structure while restricting private updates~\cite{dockhorn2022differentially, ghalebikesabi2023differentially} and further reduce privacy cost through techniques such as semantic-aware pretraining~\cite{li2024privimage}, curriculum-based schedules~\cite{wang2024dp}, or warm-starting from curated public data~\cite{park2024distribution}. Among these, DP-LDM~\cite{liu2023differentially} mitigates the curse of dimensionality by finetuning only attention modules and is one of the few methods demonstrated on high-resolution text-to-image benchmarks such as MM-CelebA-HQ, whereas most prior work focuses on class-conditional generation. Conversely, DP autoregressive image models remain underexplored, with an initial study on Vision AutoRegressive models~\cite{shaikh2025implementing} indicating that standard DP finetuning strategies currently lag behind diffusion baselines in utility. Orthogonal to DP-SGD–based training, alternative approaches generate private synthetic images without model finetuning by using private evolution over image or textual intermediaries~\cite{lin2023differentially, wang2025synthesize}, achieving strong results when the private data distribution is sufficiently close to that of the underlying pretrained models.\\

\noindent\textbf{Spectral and cascaded image generation}. Coarse-to-fine synthesis strategies decompose generation by progressively resolving structure before details. Cascaded diffusion models, as used in Imagen~\cite{saharia2022photorealistic}, implement this spatially through a pipeline of distinct models where a base generator synthesizes a low-resolution pixel grid which is subsequently refined by separate super-resolution upsamplers. In contrast, the autoregressive spectral image tokenizer (AR-SIT)~\cite{esteves2025spectral} achieves coarse-to-fine generation within a single model by transforming images into discrete wavelet coefficients ordered by frequency scale. Unlike raster-order token generation that is typical of autoregressive transformers like Parti~\cite{yu2022scaling} and LlamaGen~\cite{sun2024autoregressive}, AR-SIT generates global low-frequency components before high-frequency textures and supports partial decoding from coarse token prefixes. This disentanglement provides a natural framework for our work by allowing us to focus private learning on low-frequency semantic structure, which we hypothesize carries most of the privacy cost, while handling the high-frequency details separately.

\subsection{Notation}

\noindent\textbf{Notation and basic definitions}.
Let $\mathbb{R}$ denote the set of real numbers, $\mathbb{C}$ denote the set of complex numbers, and $\mathbb{Z}$ denote the set of integers.
We denote the 2D spatial index by $\mathbf{k} = (k_1, k_2) \in \mathbb{Z}^2$ and $\mathbb{R}^n = \mathbb{R} \times \cdots \times \mathbb{R}$, where the product is taken $n$ times. For any set $S$, let $\overline{S}$ denote its closure.

Given $X\subseteq \mathbb{R}^n$, let $L^2(X)$ be the Hilbert space of square-integrable measurable functions $f: X\to \mathbb{C}$ satisfying 
\[
\|f\|_2 = \left( \int_{-\infty}^{\infty} |f(t)|^2 dt \right)^{1/2} < \infty
\] 
with the standard inner product $\langle f, g \rangle = \int_{-\infty}^{\infty} f(t) \overline{g(t)} dt$.
Given vector spaces $V$ and $W$, let $V \otimes W$ denote their tensor product and $V \oplus W$ denote their direct sum.

For a 2D image (or subband) $X$ and a filter kernel $K$, the discrete 2D convolution $(X \star K)$ at a specific position $(i, j)$ is the sum of the element-wise product of the kernel and the overlapping image region:
\[
(X \star K)[i, j] = \sum_{m} \sum_{n} X[i-m, j-n] \cdot K[m, n].
\]
In the special case where the filter kernel is an outer product $K=uv^\top$, the above convolution can be broken down into two sequential 1D convolutions: (1) convolve the rows of $X$ with the horizontal vector $v^\top$ and then (2) convolve the columns of the resulting image with the vertical vector $u$. \\

\section{Preliminaries}
\label{sec:prelim}

This section provides a review of the following key concepts that are relevant to our work: differential privacy, MRA and 2D wavelets, and popular autoregressive image-generation methods.

\subsection{Differential privacy and DP-SGD}
\label{subsec:dp_review}

Differential privacy (DP) is a rigorous mathematical framework for quantifying the privacy leakage of randomized algorithms. It requires the output of a training procedure to be statistically insensitive to the presence or absence of any single record in the training dataset. This implies that an adversary analyzing a trained model cannot distinguish whether it was trained on a dataset $D$ or a neighboring dataset $D'$ that differs by exactly one example. Some common measures of DP are $\varepsilon$-pure DP or $(\varepsilon,\delta)$-approximate DP~\cite{dwork2006our}. A definition of the latter is given below.
\begin{definition}
($(\varepsilon,\delta)$-DP) Given $\varepsilon>0$ and $\delta \leq 1$, a mechanism (randomized algorithm) $\mathcal{A}:{\mathcal{D}}\mapsto\mathcal{R}$ is $(\epsilon, \delta)$-DP if, for any two neighboring\footnote{Differing by the addition or removal of one training example.} datasets $D$ and $D'$, and any $S \subseteq \mathcal{R}$, we have
\[ P[\mathcal{A}(D) \in S] \leq e^{\varepsilon} \cdot P[\mathcal{A}(D') \in S] + \delta. \]
\end{definition}

In this formulation, the parameter $\varepsilon$, called the privacy budget, limits how much the output distribution can change between neighbors. Smaller values of $\varepsilon$ correspond to stronger privacy guarantees and any mechanism that is $(\varepsilon_0,\delta)$-DP is also $(\varepsilon_1,\delta)$-DP for all $\varepsilon_0 \le \varepsilon_1$. The parameter $\delta$ quantifies a small probability of privacy failure and is typically chosen to be negligible, for example, satisfying $\delta \ll 1/|D|$.\\

\noindent\textbf{DP-SGD}. To achieve this guarantee in machine learning, it is common to use differentially private stochastic gradient descent (DP-SGD)~\cite{abadi2016deep}, which we give in Algorithm~\ref{alg:dpsgd}. 

\begin{algorithm}[htbp]
\caption{DP-SGD}
\label{alg:dpsgd}
\begin{algorithmic}[1]
\Require{dataset $D=\{x_1, \dots, x_N\}$, loss $\mathcal{L}$, learning rate $\eta$, noise multiplier $\sigma$, clipping norm $C$, batch size $B$, iteration count $T$}
\Ensure{DP parameters $\theta_T$}
\State Initialize $\theta_0$ randomly 
\For{$t \in 1 \dots T$}
\State \textbf{sample} Poisson batch $B_t$ with rate $q = B/N$ 
\For{each $x_i \in B_t$}
\State $g_t(x_i) \leftarrow \nabla_{\theta_t} \mathcal{L}(x_i, \theta_t)$ 
\State $\bar{g}_t(x_i) \leftarrow g_t(x_i) / \max\left(1, \|g_t(x_i)\|_2/C\right)$ 
\State \textbf{sample} $Z_t \sim \mathcal{N}(0, \sigma^2 C^2 I)$
\State $\tilde{g}_t \leftarrow ({1}/{B}) \left[ \sum_{i \in B_t} \bar{g}_t(x_i) + Z_t \right]$ 
\State $\theta_{t+1} \leftarrow \theta_t - \eta \tilde{g}_t$ 
\EndFor
\EndFor
\State \Return $\theta_T$
\end{algorithmic}
\end{algorithm}

DP-SGD modifies standard stochastic gradient descent (SGD) by enforcing a bounded global sensitivity and introducing noise to the gradient updates. 
Specifically, the influence of any training example is limited by clipping the $l_2$ norm of its gradient to a maximum threshold $C$ (line~6). Then, it is privatized by adding Gaussian noise with standard deviation $\sigma C$ to the aggregated batch gradient (line~8), which obfuscates any single example's contribution. Other variants of Algorithm~\ref{alg:dpsgd} may modify the update step in line~9 to use another first-order gradient update scheme, e.g., Adam or Adagrad~\cite{ponomareva2023dp}.

For a given noise multiplier $\sigma$, batch size $B$  iteration count $T$, and dataset size $N$, a run of DP-SGD with these parameters satisfies a cumulative privacy level of $(\varepsilon,\delta)$-DP. Privacy accountants quantify this privacy level, e.g., the value of $\delta$ and $\varepsilon$, with the tightness of the resulting bound determined by underlying modeling assumptions. In practice, users typically fix a target privacy level $(\epsilon, \delta)$ and use the accountant to solve for the minimum noise multiplier $\sigma$ required to maintain it. In our experiments, we adopt Privacy Loss Distribution (PLD) privacy accounting~\cite{doroshenko2022connect}.\\

\noindent\textbf{Post-processing}. A fundamental property of DP is its robustness to post-processing, which  guarantees that computations performed on the output of a private mechanism cannot weaken its privacy guarantees provided they do not access the original private data. Formally, this is stated as follows:

\begin{proposition}[Post-Processing~\cite{dwork2014algorithmic}]
\label{prop:ppp}
Let $\mathcal{M}: \mathcal{D} \to \mathcal{R}$ be an $(\epsilon, \delta)$-DP algorithm, and let $f: \mathcal{R} \to \mathcal{R}'$ be an arbitrary (possibly randomized) mapping. Then the composition $f \circ \mathcal{M}: \mathcal{D} \to \mathcal{R}'$ is $(\epsilon, \delta)$-DP.
\end{proposition}

This property is important in our work as it allows fine-scale image refinement to be performed entirely as post-processing of a DP coarse representation without incurring additional privacy loss.

\subsection{MRA and 2D Wavelets}
\label{subsec:wavelet_review}

We ground our coarse-to-fine image generation framework in Multiresolution Analysis (MRA), which allows us to explicitly decouple the ``private'' global structure from ``public'' local texture by decomposing the image space into nested subspaces.

In this framework, an \textit{MRA} of $L^2(\mathbb{R})$ consists of a sequence of nested\footnote{This also means that any function that is representable in $V_j$ is also representable in $V_{j+1}$.} closed function subspaces $V_j \subset V_{j+1}$ whose union is dense\footnote{That is, any square integrable function can be approximated to arbitrary precision using elements in $\bigcup_j V_j$.} in $L^2(\mathbb{R})$.
The subspace $V_0$ is defined as the closure of the linear span of integer translations of a scaling function $\phi \in L^2(\mathbb{R})$ (often called the Father wavelet). That is,
\[
V_0 = \overline{\text{span} \{\phi(\cdot -k) : k \in \mathbb{Z} \}}
\]
where $\{\phi(\cdot-k) \}_{k \in \mathbb{Z}}$ forms an orthonormal basis for $V_0$. An associated wavelet function (or Mother wavelet) $\psi$ is then constructed such that its integer translations span the orthogonal complement $W_0$ and we then construct $V_1 = V_0 \oplus W_0$.
For general scales $j$, the subspaces $V_j$ and $W_j$ are generated by dilations of $\phi$ and $\psi$ respectively, satisfying the recursive relation $V_{j+1} = V_j \oplus W_j$.

To extend this to an MRA of 2D images in $L^2(\mathbb{R}^2)$, we construct a separable basis using tensor products of the 1D spaces. This process naturally yields four generating functions that define the geometry of the 2D decomposition: one  functions scaling function $\Phi$ and three distinct wavelets $\{\Psi^H, \Psi^V, \Psi^D\}$ where $H =$ horizontal, $V=$ vertical, and $D$ = diagonal.
The 2D approximation space at scale $j$, denoted $V_j^{2D}$, is then the tensor product $V_j \otimes V_j$. Crucially, the detail space needed to move from scale $j$ to $j+1$ decomposes into three orthogonal orientation bands:
\[
V_{j+1}^{2D} = \underbrace{(V_j \otimes V_j)}_{V_j^{2D}} \oplus \underbrace{(V_j \otimes W_j)}_{W_j^H} \oplus \underbrace{(W_j \otimes V_j)}_{W_j^V} \oplus \underbrace{(W_j \otimes W_j)}_{W_j^D}.
\] 
Similar to the 1D setting, the orthonormal basis for the subspace $V_j^{2D}$ consists of the collection of all integer translations of the dilated generator. More specifically, for a spatial index $\mathbf{k} = (k_1, k_2) \in \mathbb{Z}^2$ and scale $j$, the normalized 2D scaling basis function is
\begin{equation}
    \Phi_{j,\mathbf{k}}(x,y) = 2^j \Phi(2^j x - k_1, 2^j y - k_2) \label{eq:2d_scale_fn}
\end{equation}
and the set $\{ \Phi_{j,\mathbf{k}} \}_{\mathbf{k} \in \mathbb{Z}^2}$ forms the complete orthonormal\footnote{The factor $2^j$ is required to maintain unit energy (normalization). As scale $j$ increases, the spatial support shrinks by $2^{-j}$ in both dimensions (area scales by $2^{-2j}$). To satisfy the isometry condition $\|\Phi_{j,\mathbf{k}}\|_2 = 1$, the amplitude must scale by $\sqrt{2^{2j}} = 2^j$.} basis for $V_j^{2D}$. Similarly, we can define $\Psi_{j,k}^\lambda$ to be as in \eqref{eq:2d_scale_fn} but with $\Phi=\Psi^\lambda$ and the sets $\{ \Psi^\lambda_{j,\mathbf{k}} \}_{\mathbf{k} \in \mathbb{Z}^2}$ form the bases for the detail spaces $W_j^\lambda$ for $\lambda \in \{H,V,D\}$. 

Consequently, using the orthogonality of the basis functions, for a given \textit{coarseness level} (hyperparameter) $j_0 \geq 1$, any image $f(x,y)$ can be expressed as:
\begin{align*}
f(x,y) &= \underbrace{\sum_{\mathbf{k} \in \mathbb{Z}^2} c_{j_0, \mathbf{k}} \Phi_{j_0, \mathbf{k}}(x,y)}_{\text{Approximation at scale } j_0} + \sum_{j=j_0}^{\infty} \underbrace{\sum_{\mathbf{k} \in \mathbb{Z}^2} \sum_{\lambda \in \{H,V,D\}} d^\lambda_{j, \mathbf{k}} \Psi^\lambda_{j, \mathbf{k}}(x,y)}_{\text{Details for scale }j},
\end{align*}
where the coefficients are defined by  inner products with the basis functions:
\begin{equation}
c_{j_0, \mathbf{k}} = \langle f, \Phi_{j_0, \mathbf{k}} \rangle \quad \text{and} \quad 
d^\lambda_{j, \mathbf{k}} = \langle f, \Psi^\lambda_{j, \mathbf{k}} \rangle. 
\end{equation}
This representation maps directly to our privacy hypothesis:
\begin{itemize}
    \item \textbf{Approximation coefficients} ($c_{j_0, \mathbf{k}}$): These coefficients project the image onto $V_{j_0}^{2D}$. They capture the bulk of the signal energy and semantic identity, \textit{representing the ``private'' payload}.
    \item \textbf{Detail coefficients} ($d^\lambda_{j, \mathbf{k}}$): These coefficients project onto the detail spaces $W_{j}^{2D}$. They are typically sparse and capture \textit{generic textures that we treat as ``public'' information}.
\end{itemize}
While formally defined by continuous integrals, the coefficients $\{c_{j,\mathbf{k}}, d_{j,\mathbf{k}}^\lambda\}$ are computed in practice using the Discrete Wavelet Transform (DWT), also known as Mallat's algorithm~\cite{mallat1999wavelet}. Given an input of $N$ pixels, this efficient $O(N)$ algorithm utilizes a cascade of discrete low-pass and high-pass filters followed by dyadic downsampling. More specifically, at each scale $j$, the DWT decomposes the signal into four distinct subbands; we denote these using the standard filter notation $XY_j$, where $X,Y \in \{L, H\}$ indicate whether the low-pass or high-pass filter was applied to the rows and columns, respectively (see Figure~\ref{fig:wavelet} for a visualization):
\begin{itemize}
    \item $LL_j$ ($\approx c_{j_0 + j,\mathbf{k}}$): Represents the approximation at scale $j_0+j$. In the encoding phase, this becomes the input for the next decomposition step.
    \item $LH_j$ ($\approx d^H_{j_0 + j,\mathbf{k}}$): Captures horizontal edges.
    \item $HL_j$ ($\approx d^V_{j_0 + j,\mathbf{k}}$): Captures vertical edges.
    \item $HH_j$ ($\approx d^D_{j_0 + j,\mathbf{k}}$): Captures corner features and diagonal textures.
\end{itemize}
In particular, given a coarseness scale $j_0$, the subband $LL_0$ corresponds to the approximation at the scale $j_0$ and is obtained by repeatedly applying the DWT. This recursive calculation is realized in our experiments using the Haar wavelet basis~\cite{haar1910theorie}; see \eqref{eq:haar_update} below. Although wavelet coefficients are formally defined as inner products over continuous domains, the DWT computes them efficiently via discrete convolution operations. As a result, extracting the $LL_0$ tokens required for private training amounts to running a standard DWT and avoids any high-dimensional continuous computations.

\begin{figure}[t]
    \centering
    \includegraphics[width=0.9\linewidth]{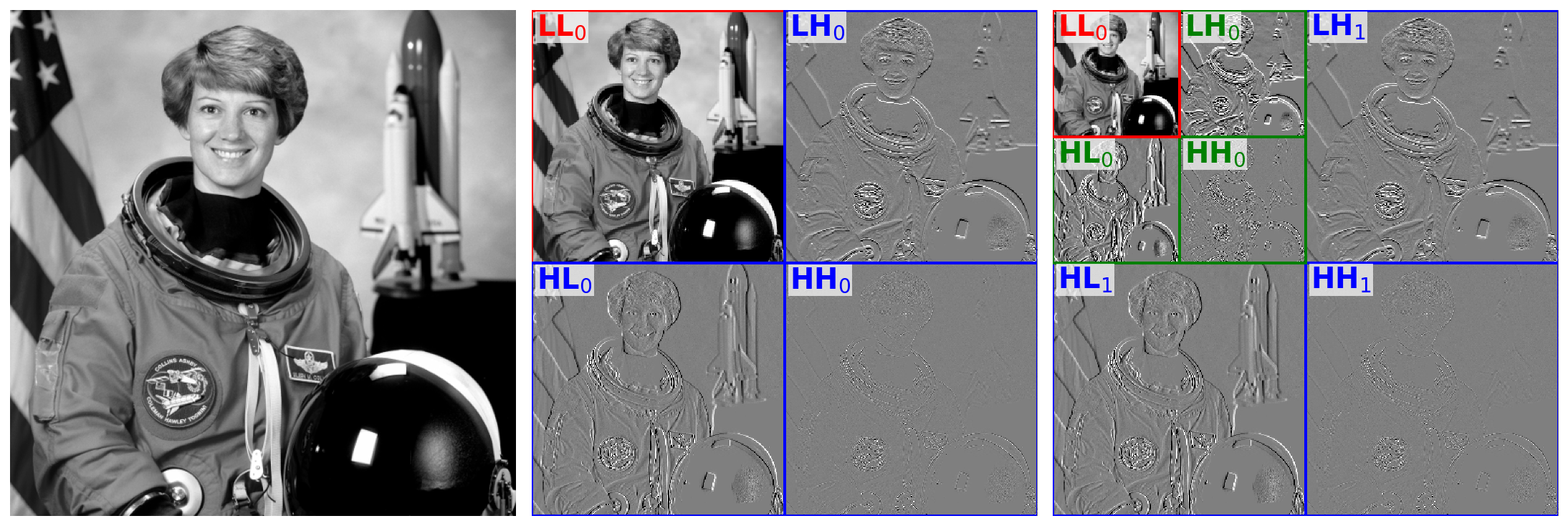}
    \caption{Visualization of the 2D Discrete Wavelet Transform. \textit{Left}. Original image, \textit{Center}. A one-level decomposition where the image is split into the approximation $LL_0$ and details $\{LH_0, HL_0, HH_0\}$, \textit{Right}, A two-level decomposition where the image is first split into scale 1 details in blue and a scale 1 approximation, which is then recursively split into scale 0 details in green and the final approximation $LL_0$ in red. Note that indices are relative to the coarsest scale $LL_0$.} 
    \label{fig:wavelet}
    \Description{Illustration of repeated DWT}
\end{figure}

\subsection{Autoregressive spectral image tokenizer}

Our generative framework builds on the popular vector-quantized autoregressive image models, specifically the VQ-GAN paradigm~\cite{esser2021taming}. In this setting, an encoder-decoder tokenizer ($\mathcal{E}$, $\mathcal{D}$) maps images to sequences of discrete tokens drawn from a learned codebook while an autoregressive transformer $\mathcal{T}$ models the token distribution. This formulation reduces high-dimensional image synthesis to sequence modeling over discrete tokens and enables the use of standard transformer architectures, as demonstrated by large-scale autoregressive models such as Parti~\cite{yu2022scaling} and LlamaGen~\cite{sun2024autoregressive}.

\begin{figure}[tbh]
    \centering
\includegraphics[width=0.7\linewidth]{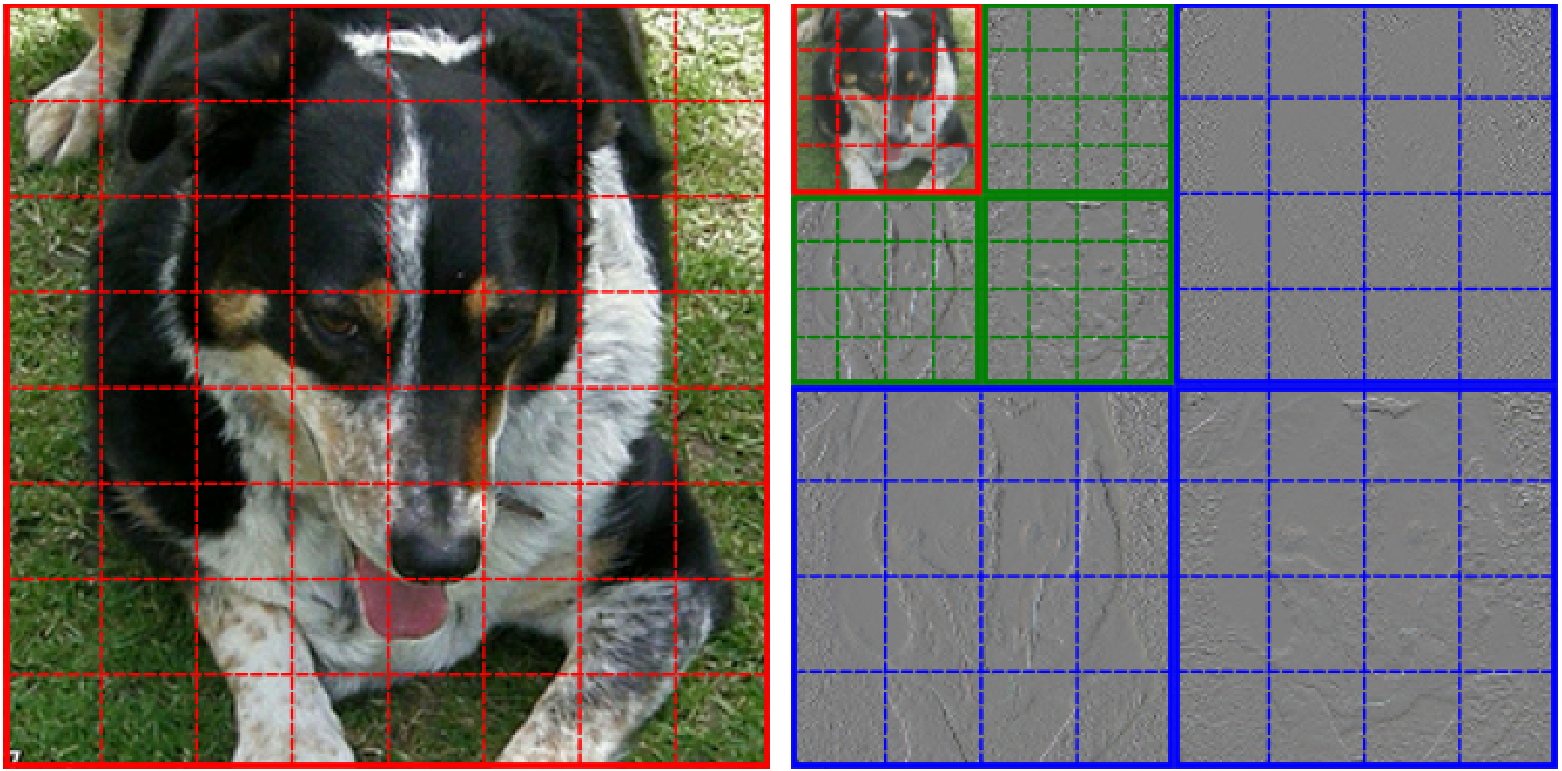}
    \caption{\textit{Left}. Pixel-based patchification, \textit{Right}. Wavelet-based patchification where separate codebooks are used for the red, green, and blue patches~\cite{esteves2025spectral}.}
    \label{fig:ar_sit}
    \Description{Patchification schemes}
\end{figure}

We use the autoregressive spectral image tokenizer (AR-SIT)~\cite{esteves2025spectral}, which extends this framework by performing tokenization in the wavelet domain rather than pixel space. AR-SIT applies a discrete wavelet transform (DWT) to decompose an image into multiscale subbands and then patchifies each subband independently. In our implementation, we use Haar wavelets for their simplicity and strong reconstruction properties. Using the notation from Section~\ref{subsec:wavelet_review}, the subbands are computed recursively as

\begin{equation}
\begin{gathered}
LL_{j-1} = LL_{j} \star (gg^\top), \quad
LH_{j-1} = LL_{j} \star (gh^\top), \\
HL_{j-1} = LL_{j} \star (hg^\top), \quad 
HH_{j-1} = LL_{j} \star (hh^\top), \nonumber
\end{gathered}
\label{eq:haar_update}
\end{equation}
for every $j\geq 1$, where $g=[1,1]^\top$, $h=[1,-1]^\top$, and $\star$ is the discrete 2D convolution.

The AR-SIT encoder $\mathcal{E}$ applies the DWT up to a specified coarseness level and patchifies every wavelet subband independently; Figure~\ref{fig:ar_sit} shows the visual contrast between this spectral patching process and standard pixel/raster-order patching. Since approximation and detail subbands have different statistics across scales, AR-SIT uses an Approximation–Details Transformer (ADTransformer) that processes approximation and detail tokens with separate layer normalizations, MLPs, and attention blocks rather than shared parameters.

A key capability of this architecture is partial decoding, which is achieved by enforcing scale-causal attention in the SIT decoder. This mechanism ensures that the initial sequence of tokens corresponding to the low-frequency approximation $LL_0$ can be decoded into a recognizable coarse image before any high-frequency details are generated. That is, for a maximum recursion depth $J$ and $0< k \leq J$ the tokens for $LL_{0}$ and $\{LH_{j}, HL_{j}, HH_{j}\}_{j=0}^{k}$ may be used as an approximation of the fully re-constructed image; see Figure~\ref{fig:ar_sit_partial} for an illustration and \cite{esteves2025spectral} for more examples. This separability is central to our framework because it allows us to restrict the privacy budget to the coarse generation stage while delegating detail synthesis to a frozen public prior.

\begin{figure}[tbh]
    \centering
    \includegraphics[width=0.9\linewidth]{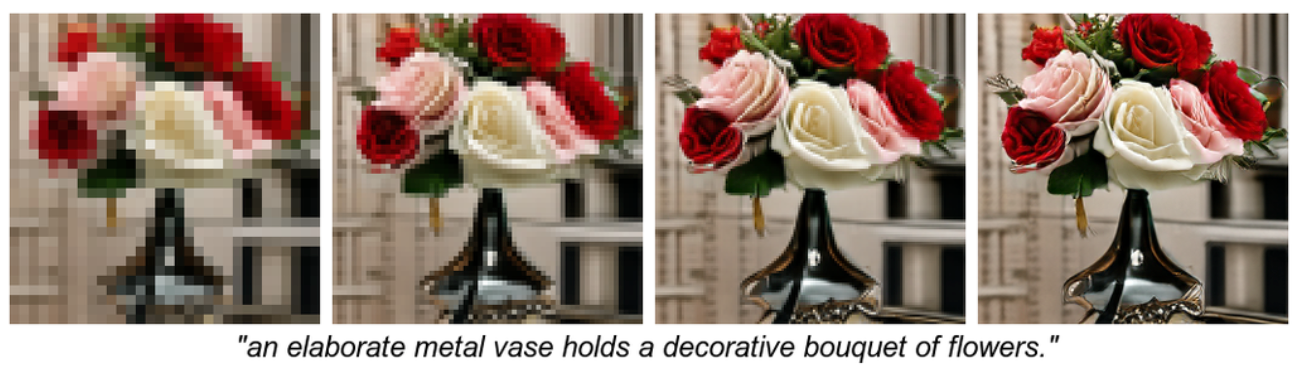}
    \caption{An example of AR-SIT's partial generation of images using 25\%, 50\%, 75\%, and 100\% (left-to-right) of the full AR transformer token sequence~\cite{esteves2025spectral}.}
    \label{fig:ar_sit_partial}
    \Description{Illustration of partial rendering in AR-SIT}
\end{figure}

\section{DP-Wavelet for DP image generation}
\label{sec:dp_wavelet}

In this section, we present DP-Wavelet, a coarse-to-fine framework for differentially private text-to-image generation. As outlined in Algorithm~\ref{alg:dp_wavelet}, our method decomposes the generative process into two distinct stages: private coarse-scale synthesis and public fine-scale completion. In the first stage, we finetune a pretrained AR-SIT transformer on a private dataset using a DP optimizer (e.g., DP-Adam, DP-SGD, or DP-LoRA). Crucially, to maximize the gradient's signal-to-noise ratio, this optimization is limited to the parameters responsible for predicting the energy-dense coarse wavelet tokens (scale $j_0$) conditioned on text. In the second stage, we sample the DP-finetuned model to generate DP coarse tokens and autoregressively predict the remaining high-frequency tokens with the frozen, publicly pretrained transformer. The resulting complete spectral sequence is then decoded into the output image.

\begin{algorithm}[htbp]
\caption{DP-Wavelet}
\label{alg:dp_wavelet}
\begin{algorithmic}[1]
\Require 
    Coarseness scale $j_0$; dataset of text-image pairs $\{(x_i, y_i)\}$; public pretrained AR-SIT encoder $\mathcal E$, decoder $\mathcal D$, and transformer $\mathcal T$; privacy parameters $(\varepsilon, \delta)$; 
\Ensure DP images $\{\hat{y}_i\}$.

\Statex \textbf{Stage 1: Low resolution DP finetuning}
\State \textit{Tokenize} each private image $y_i$ using $\mathcal E$ to extract the sequence of approximation tokens $\{t_{i,k}\}_{k\geq 1}$ corresponding to the $LL_{0}$ subband at scale $j_0$.

\State \textit{DP Finetune} $\mathcal T$ using a privacy budget $(\varepsilon, \delta)$ to obtain $\hat{\mathcal T}$, optimizing for the prediction of partial sequences $t_i$ conditioned on texts $x_i$.

\Statex \textbf{Stage 2: Inference and upsampling}

\State \textit{Generate} coarse DP token sequences $\{\hat{t}_{i,k}\}_{k\geq1}$ using $\hat{\mathcal T}$ conditioned on public texts $\{x_i\}$.

\State \textit{Complete} (upsample) the partial sequences into full token sequences $\{s_{i,k}\}_{k\geq 1}$ using the frozen pretrained transformer $\mathcal T$.

\State \textit{Decode} the completed sequences $\{s_{i,k}\}_{k\geq 1}$ using $\mathcal D$ to obtain images $\hat y_i$.
\State \Return $\{\hat{y}_i\}$
\end{algorithmic}
\end{algorithm}

The design of DP-Wavelet is motivated as follows.

\begin{itemize}
    \item \textbf{Mitigating the curse of dimensionality}. 
    The utility cost of DP optimization is known to scale with the square root of the parameter dimension $d$~\cite{bassily2014private}. Motivated by this, we strictly limit finetuning in stage 1 to the subset of AR transformer parameters in $\mathcal{T}$ responsible for generating coarse scale $j_0$ tokens. By freezing the parameters governing high-frequency details in addition to the tokenizer $\mathcal{E}$ and decoder $\mathcal{D}$, we significantly reduce the dimensionality of the gradient updates and consequently the utility cost.

    \item \textbf{Maximizing signal-to-noise via energy compaction.} By the energy compaction property of wavelets, the majority of an image's signal energy lies in the low-frequency $LL_{0}$ subband~\cite{mallat2008wavelet}. This motivates our strategy of DP finetuning on the coarse scale as gradients derived from these energy-dense coefficients likely carry a more robust learning signal relative to the fixed noise floor of the DP optimizer. In contrast, gradients for high-frequency subbands are typically sparse and low-magnitude, making them more susceptible to being overwhelmed by DP noise.

    \item \textbf{Universality of high-frequency priors.} We hypothesize that high-frequency subbands \\
    $\{LH_{j}, HL_{j}, HH_{j}\}_{j\geq 0}$ capture domain-agnostic texture statistics~\cite{simoncelli2001natural}. Under this assumption, a model pretrained on public data can synthesize these details without accessing private data. This enables us to treat detail generation as a post-processing step and reserve the full privacy budget for structural components.

\end{itemize}

In Appendix~\ref{app:dp_proof}, we show Algorithm~\ref{alg:dp_wavelet} is $(\varepsilon,\delta)$-DP.

\section{Experiments}
\label{sec:numerical}

We evaluate DP-Wavelet on the MS-COCO dataset~\cite{lin2014microsoft}, containing diverset text-image pairs from the internet, and the MM-CelebA-HQ dataset~\cite{xia2021tedigan}, consisting of celebrity faces with descriptive captions.

\noindent\textbf{Baselines.} We compare our method against two state-of-the-art private generative models:

\begin{itemize}
\item \textbf{DP-LDM}~\cite{liu2023differentially} is a latent diffusion model finetuned from the Stable Diffusion v1.4 checkpoint. Following the  procedure in \cite{liu2023differentially}, DP finetuning is applied only to the U-Net attention modules and text conditioning parameters.
\item \textbf{DP-LlamaGen} is our DP adaptation of the LlamaGen-XL autoregressive model~\cite{sun2024llamagen}, which generates images as raster-ordered sequences of discrete image tokens. We apply DP finetuning to all transformer parameters, as the raster ordering interweaves structure and detail, preventing the clean spectral separation used in DP-Wavelet.
\end{itemize}

\noindent\textbf{Implementation details.} For DP-Wavelet, we use the pretrained AR-SIT tuple $(\mathcal E, \mathcal D, \mathcal T)$ corresponding to the ``SIT-SC-5'' configuration in \cite{esteves2025spectral}. We set the coarseness scale $j_0$ to correspond to a $32 \times 32$ spatial resolution. Privacy accounting is performed using the PLD accountant in JAX Privacy~\cite{jaxprivacy2022github}. Additional hyperparameters and training details are provided in Appendix~\ref{app:hypers}.

\noindent\textbf{Evaluation Metrics}. We benchmark performance across three privacy regimes: (i) \textit{Pretrained} with zero-shot inference, (ii) \textit{Non-private} with $\varepsilon = \infty$, and (iii) \textit{Private} with $\varepsilon < \infty$. We report two standard metrics computed on 30k samples:
\begin{itemize}
\item \textbf{Fréchet Inception Distance (FID)}~\cite{heusel2017gans} measures the distributional similarity between generated and real images (lower is better).
\item \textbf{LPIPS}~\cite{zhang2018unreasonable} measures the average perceptual distance between ground truth images and generated samples conditioned on the same text prompt (lower is better).
\end{itemize}

\subsection{MS-COCO results}
\label{subsec:ms_coco}

\noindent\textbf{Data processing.} We finetune all models on the MS-COCO training split using random cropping and horizontal flipping data augmentations. For text conditioning, we sample uniformly from the five captions available per image during training. To ensure the LPIPS evaluation is valid during inference, we consistently condition the generation on the first caption of each test image.

\noindent\textbf{Quantitative results.} Table~\ref{tab:ms_coco_metrics} presents the performance across the three privacy regimes for each method, from which we observe:
\begin{itemize}
\item \textit{Base model bias.} DP-LDM achieves the lowest  FID scores across most settings. A likely contributing factor is the Stable Diffusion backbone's pretraining on LAION-5B (5.85 billion image-text pairs), which subsumes the MS-COCO distribution. Similarly, DP-LlamaGen demonstrates strong performance, surpassing DP-LDM in LPIPS values, and also benefits from its pretraining on the LAION-COCO subset (50 million pairs), which was curated to match the COCO distribution. In contrast, the AR-SIT model was pretrained on a subset of 128 million images from WebLI~\cite{chen2022pali}, resulting in a greater distributional mismatch that contributes to its higher starting FID relative to the baselines.

\item \textit{Finetuning performance.} In the non-private regime ($\varepsilon=\infty$), DP-Wavelet achieves the best overall LPIPS of 0.666, validating the efficacy of spectral tokenization for structural modeling. This property remains robust under privacy budgets ($\varepsilon=10, 1$), where our method maintains competitive LPIPS values comparable to the DP-LlamaGen baseline.
\end{itemize}

\begin{table}[htbp]
\caption{Performance on MS-COCO. \label{tab:ms_coco_metrics}}
\begin{center}
\begin{tabular}{l|ccc}
 & \multicolumn{3}{c}{\textbf{FID} $\downarrow$ / \textbf{LPIPS} $\downarrow$} \\ \cline{2-4} 
 & DP-LDM & DP-LlamaGen & DP-Wavelet \\ \hline
Pretrained & \textbf{13.2} / 0.759 & 17.7 / \textbf{0.728} & 23.1 / 0.751 \\ 
$\varepsilon = \infty$ & 12.2 / 0.755 & \textbf{11.6} / 0.715 & 15.5 / \textbf{0.666} \\ 
$\varepsilon = 10$ & \textbf{12.3} / 0.757 & 13.8 / \textbf{0.723} & 19.4 / 0.751 \\ 
$\varepsilon = 1$ & \textbf{12.8} / 0.758 & 16.1 / \textbf{0.728} & 20.1 / 0.749 \\ 
\end{tabular}
\end{center}
\end{table}

\noindent\textbf{Qualitative results.} Figure~\ref{fig:ms_coco_samples} presents samples across the three privacy regimes. Consistent with the metrics in Table~\ref{tab:ms_coco_metrics}, DP-Wavelet demonstrates strong adaptability in the non-private setting ($\varepsilon=\infty$), faithfully capturing specific attributes such as the dress color and background composition in the first row. As the privacy budget tightens ($\varepsilon < \infty$), these specific details shift towards more generic representations (e.g., changing the dress color and background), demonstrating effective privacy protection while maintaining semantic alignment. In contrast, DP-LDM generates nearly identical outputs across some regimes (e.g., first row), suggesting an over-reliance on its public pretraining.

\begin{figure}[htbp]
    \centering
    \begin{subfigure}[t]{0.49\linewidth}
        \centering
        \includegraphics[width=\linewidth]{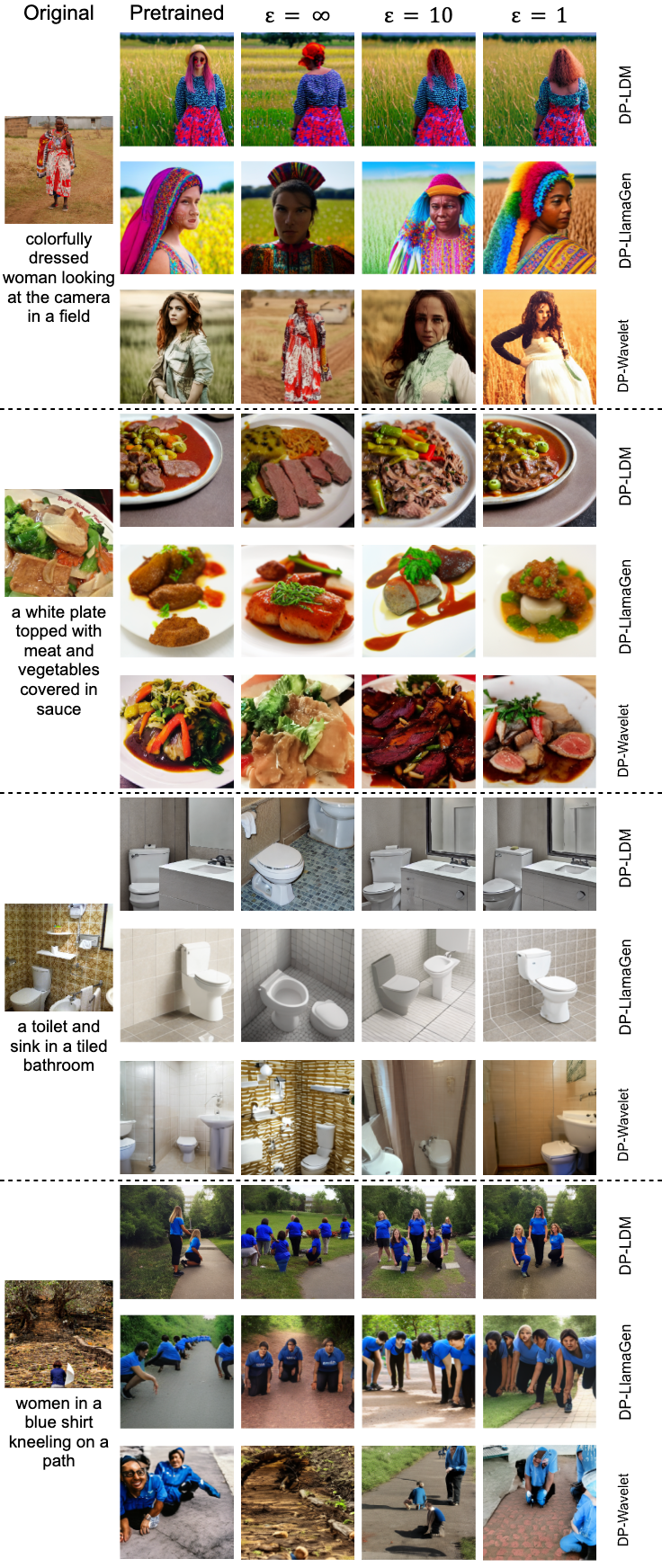}
        \caption{MS-COCO samples}
        \label{fig:ms_coco_samples}
    \end{subfigure}
    \hfill
    \begin{subfigure}[t]{0.49\linewidth}
        \centering
        \includegraphics[width=\linewidth]{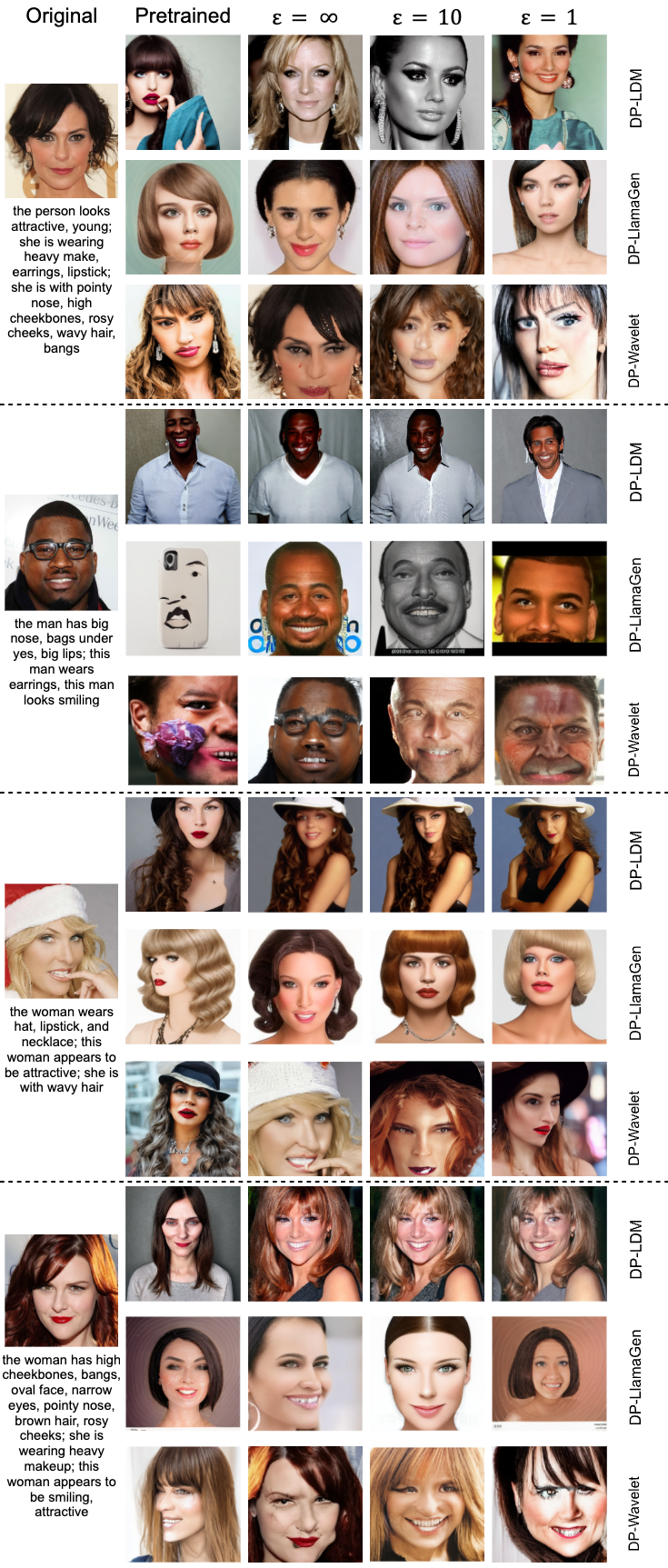}
        \caption{MM-CelebA-HQ samples}
        \label{fig:mm_celeb_samples}
    \end{subfigure}
    \caption{Sample images generated by the benchmark algorithms for MS-COCO and MM-CelebA-HQ.}
    \label{fig:sample_images_side_by_side}
\end{figure}

\subsection{MM-CelebA-HQ results}
\label{subsec:mm_celeb}

\noindent\textbf{Data processing.} We finetune all models on the MM-CelebA-HQ dataset using the same image transforms and taking the sole caption for each image to form the text-image pairs.

\noindent\textbf{Quantitative results.} Table~\ref{tab:mm_celeb_metrics} presents the FID and LPIPS values for the different privacy regimes on MM-CelebA-HQ, from which we observe:
\begin{itemize}
\item \textit{Base model bias.} Pretraining bias is likely less pronounced on MM-CelebA-HQ as none of the base models were explicitly trained on face-centric datasets. In this setting, DP-Wavelet attains the lowest pretrained FID, suggesting that our coarse-to-fine spectral prior generalizes more effectively to new domains than the standard autoregressive or diffusion baselines.

\item \textit{Finetuning performance.} DP-Wavelet demonstrates the most stable performance across privacy regimes. In the non-private case ($\varepsilon=\infty$) it achieves the best metrics and under moderate privacy ($\varepsilon=10$) it attains the lowest FID value and LPIPS close to DP-LlamaGen. In contrast, DP-LlamaGen suffers from a degradation of FID during DP finetuning and the stagnant LPIPS of DP-LDM suggests limited domain adaptation.
\end{itemize}

\begin{table}[htbp]
\caption{Performance on MM-CelebA-HQ.\label{tab:mm_celeb_metrics}}
\begin{center}
\begin{tabular}{l|ccc}
 & \multicolumn{3}{c}{\textbf{FID} $\downarrow$ / \textbf{LPIPS} $\downarrow$} \\ \cline{2-4} 
 & DP-LDM & DP-LlamaGen & DP-Wavelet \\ \hline
Pretrained & 64.7 / 0.734 & 86.5 / \textbf{0.672} & \textbf{54.5} / 0.711 \\ 
$\varepsilon = \infty$ & 25.7 / 0.726 & 32.5 / 0.586 & \textbf{18.2} / \textbf{0.581} \\ 
$\varepsilon = 10$ & 22.9 / 0.719 & 63.7 / \textbf{0.649} & \textbf{22.2} / 0.659 \\ 
$\varepsilon = 1$ & \textbf{25.9} / 0.726 & 74.2 / \textbf{0.659} & 39.4 / 0.678 \\ 
\end{tabular}
\end{center}
\end{table}

\noindent\textbf{Qualitative results.}
In Figure~\ref{fig:mm_celeb_samples}, we provide image samples generated by models for each privacy regime. As on the MS-COCO dataset, DP-Wavelet in the non-private setting demonstrates strong semantic adaptation, faithfully capturing detailed attributes of the reference image such as skin tone, hairstyle, and accessories (e.g., glasses for the second image). In the private settings, our method maintains a better balance of semantic fidelity and image quality than the baselines. Notably, images generated under strict privacy ($\varepsilon=1$) reveal more visual artifacts as a consequence of the lower signal-to-noise ratio inherent to the DP training regime.

\subsection{Discussion}

Our results suggest that DP-Wavelet's advantages stem from concentrating the privacy budget on coarse structural prediction ($LL_0$) rather than distributing it across all image components. DP-Wavelet demonstrates strong performance on the challenging MM-CelebA-HQ dataset which is likely due to the improved signal-to-noise ratio in gradient updates that comes from only updating parameters corresponding to high-energy coefficients. However, on MS-COCO there is some gap in image quality between the images generated by DP-LDM and DP-LlamaGen. To investigate this, we visualized the decoded coarse image intermediaries from DP-Wavelet, as shown in Appendices \ref{app:lowres_samples} and \ref{app:more_dp_wavelet}, and found that the coarse tokens already capture much of the global image structure and color. This implies that the remaining image quality gap arises from the frozen public upsampler which may be fixed with stronger pretraining. Beyond utility, the spectral decomposition also improves computational efficiency: by restricting DP gradient updates to the low-dimensional coarse scale, we  reduce the optimization burden and achieve faster training times as discussed in Appendix \ref{app:train_speed}.

Regarding the weaker domain adaptation observed for DP-LDM, we hypothesize that text conditioning plays a different role across methods. For diffusion-based and raster-order autoregressive baselines, short captions may provide an insufficient signal to recover the style and color from private images, especially under DP noise. Since DP-Wavelet demonstrates strong domain adaptation and relies less on textual specificity by anchoring generation to coarse image structure, this suggests that applying DP at the level of semantically meaningful image intermediaries is particularly effective.

\section{Conclusion}
\label{sec:conclusion}

In this work, we introduced DP-Wavelet, a coarse-to-fine framework for differentially private text-to-image generation that applies privacy at the level of semantically meaningful image intermediaries. By decomposing image synthesis into a private coarse-scale prediction and a public fine-scale refinement, DP-Wavelet can allocate the privacy budget to low-frequency wavelet components that capture global structure, while delegating high-frequency detail generation to a frozen public prior via DP post-processing.

Our approach is the first scalable autoregressive method for text-to-image generation under differential privacy, demonstrating that autoregressive models can be competitive with a strategic DP finetuning method. Empirically, DP-Wavelet achieves strong style preservation and competitive image quality on challenging benchmarks, particularly on MM-CelebA-HQ, while remaining more computationally efficient than diffusion-based baselines.

More broadly, our results suggest that improving private generative modeling requires not only better DP optimizers, but also careful alignment between the privacy mechanism and the structure of the data representation. We view coarse image intermediaries as a promising direction for future work on high-utility, privacy-preserving generative models.

\bibliographystyle{abbrvnat}
\bibliography{references}
%%%%%%%%%%%%%%%%%%%%%%%%%%%%%%%%%%%%%%%%%%%%%%%%%%%%%%%%%%%%

\newpage
\appendix

\section{Training summary and hyperparameters}

This section presents additional training environment details, including learning rate (LR), the L2 clip norm bound used for DP-SGD, and the number of linear warm-up steps used in the LR schedule.

All methods employ a learning rate schedule with linear warm-up and cosine decay (down to 5\% of the maximum learning rate) and run DP-Adam as the DP optimizer. 

DP-Wavelet and DP-LDM are implemented in JAX (Python)~\cite{jax2018github} with TPU hardware acceleration, using Kauldron~\cite{kauldron2024} as a configuration manager. DP-LlamaGen is implemented in PyTorch with GPU hardware acceleration. For hardware, we use 32 TPU v5p chips for finetuning DP-Wavelet, 64-128 TPU v5p chips for DP-LDM (depending on the dataset), and one H100 GPU for DP-LlamaGen. 

One point of interest is that the training speed of DP-Wavelet is faster when DP-Adam is used instead of Adam. We conjecture that the XLA compilation of the DP-Adam training loop provides is more efficient at reducing bottlenecks, e.g., through operator fusion and static shape inference, than the XLA compilation of (non-DP) Adam's training loop.

Table~\ref{tab:ms_coco_hypers} and \ref{tab:mm_celeb_hypers} present the training hyperparameters for the MS-COCO and MM-CelebA-HQ experiments, respectively. It is worth mentioning that the number of training steps was chosen at the point where either the loss stopped decreasing or the FID stopped changing.

\begin{table*}[tb]
    \caption{Summary of the experimental setup.}
    \label{tab:experiment_setup}
    \centering
    \begin{tabular}{lccc}
        \toprule
        \textbf{Setting} & \textbf{DP-LDM}~\cite{liu2023differentially} & \textbf{DP-LlamaGen}~\cite{sun2024autoregressive} & \textbf{DP-Wavelet} (Ours) \\
        \midrule
        Base Architecture & Stable Diffusion v1.4 (UNet) & LlamaGen-XL & AR-SIT (SIT-SC-5) \\
        Trainable Parameters & 343M & 775M & 350M \\
        Loss Function & Mean-Squared Error & LlamaGen Training Loss\tablefootnote{See the \textbf{training losses} paragraph in Section 2.2 of \cite{sun2024autoregressive} for details.} & Softmax Cross-Entropy \\
        Compute Resources & 64-128 v5p TPUs & 1 H100 GPU  & 32 v5p TPUs \\
        Finetuning Speed $\uparrow$ & 7.14 steps/s & 0.29 steps/s & 7.78 steps/s \\
        DP-Finetuning Speed $\uparrow$ & 0.20 steps/s & 0.04 steps/s & 8.46 steps/s \\
        \bottomrule
    \end{tabular}
\end{table*}

\label{app:hypers}
\begin{table}[htbp]
\centering
\caption{MS-COCO hyperparameters\label{tab:ms_coco_hypers}.}
\begin{tabular}{l|ccc}
& DP-LDM & DP-LlamaGen & DP-Wavelet \\ \hline
Max LR for $\varepsilon = \infty$ & $1e\mbox{-}5$ & $2e\mbox{-}5$ & $2e\mbox{-}5$ \\ 
Max LR for $\varepsilon < \infty$ & $1e\mbox{-}4$ & $2 e\mbox{-}5$ & $2 e\mbox{-}5$ \\ 
\shortstack[l]{\# steps for $\varepsilon = \infty$} & 3,230 & 4,500 & 100,000 \\ 
\shortstack[l]{\# steps for $\varepsilon < \infty$} & 3,230 & 1,125 & 100,000 \\ 
\# LR warmup steps & 100 & 500 & 500 \\ 
Batch size\tablefootnote{DP-LlamaGen used a batch size of 256 for non-private finetuning.} & 256 &  1024 & 1024 \\ 
L2 Clip for $\varepsilon < \infty$ & $1e\mbox{-}1$ & $1e\mbox{-}1$ & $1e\mbox{-}1$ \\
\end{tabular}
\end{table}

\begin{table}[b]
\caption{MS-COCO hyperparameters\label{tab:mm_celeb_hypers}}
\centering
\begin{tabular}{l|ccc}
& DP-LDM & DP-LlamaGen & DP-Wavelet \\ \hline
Max LR for $\varepsilon = \infty$ & $5e\mbox{-}5$ & $2e\mbox{-}5$ & $2e\mbox{-}5$ \\ 
Max LR for $\varepsilon < \infty$ & $2e\mbox{-}5$ & $2e\mbox{-}5$ & $2e\mbox{-}5$ \\ 
\shortstack[l]{\# steps for $\varepsilon = \infty$} & 4,680 & 1,125 & 100,000 \\ 
\shortstack[l]{\# steps for $\varepsilon < \infty$} & 4,680 & 1,125 & 100,000 \\ 
\# LR warmup steps & 100 & 500 & 500 \\ 
Batch size & 1024 & 1024 & 1024 \\ 
L2 Clip for $\varepsilon < \infty$ & $1e\mbox{-}1$ & $1e\mbox{-}1$ & $1e\mbox{-}1$  \\
\end{tabular}
\end{table}

\section{Low resolution samples from DP-Wavelet}
\label{app:lowres_samples}

In Figure~\ref{fig:lowres_wavelet_samples}, we present the low-frequency subbands generated by $\hat{\mathcal{T}}$ in stage 2 of DP-Wavelet (Algorithm~\ref{alg:dp_wavelet}). Notice that the images for the $\varepsilon=\infty$ case strongly resemble the original training image, but the upsampled versions in Tables~\ref{fig:ms_coco_samples} and \ref{fig:mm_celeb_samples} introduce graphical artifacts that reduce the overall fidelity of the generate image.

\begin{figure}[htbp]
    \centering
    \includegraphics[width=0.7\linewidth]{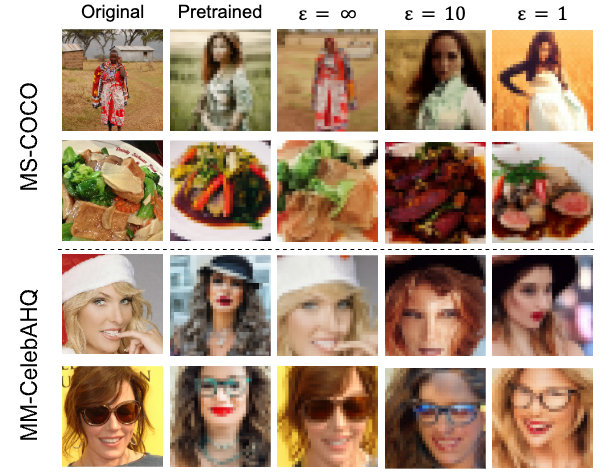}
    \caption{MS-COCO and MM-CelebA-HQ sample images generated by $\hat{\mathcal{T}}$ in stage 2 of DP-Wavelet. These are the decoded images formed by $\{\hat{t}_{i}\}$ when passing them through the decoder $\mathcal D$. \label{fig:lowres_wavelet_samples}}
    \Description{Low resolution DP-Wavelet samples}
\end{figure}

\newpage
\section{More non-private examples from DP-Wavelet}
\label{app:more_dp_wavelet}
In Figure~\ref{fig:more_dp_wavelet}, we present more examples generated by DP-Wavelet for both the MS-COCO and MM-CelebA-HQ datasets for the $\varepsilon=\infty$ setting. Each column consists of a triple containing original reference image (left), the final image generated by DP-Wavelet (center), and the low-resolution DP-Wavelet intermediary (right).

\begin{figure}[htbp]
    \centering
    \includegraphics[width=0.6\linewidth]{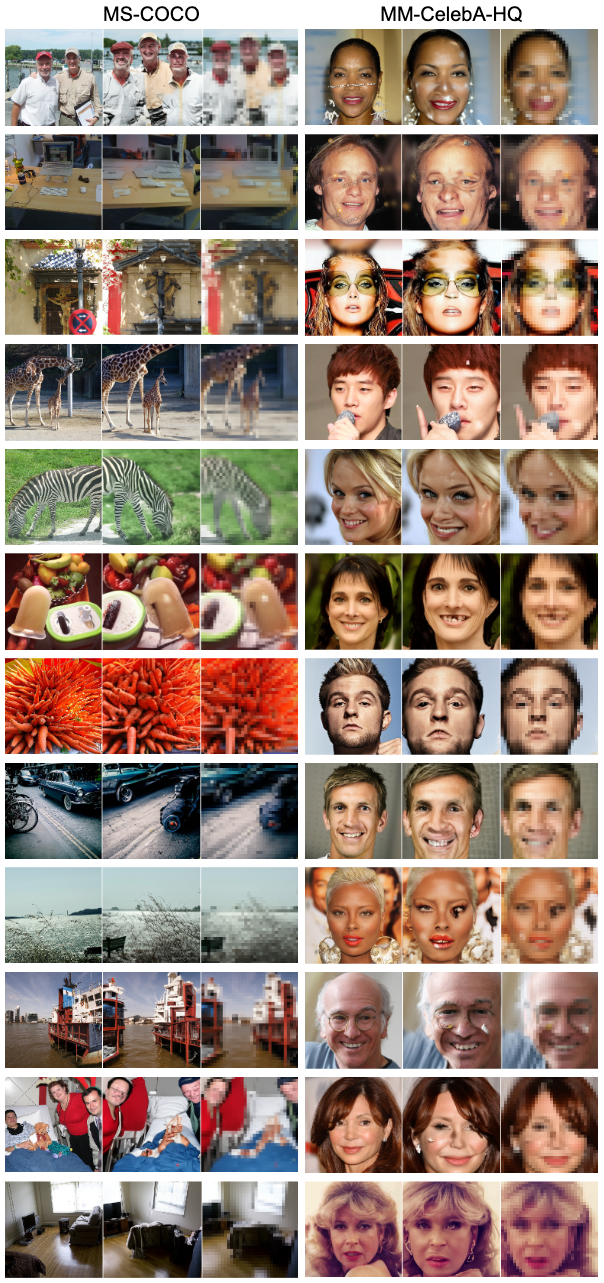}
    \caption{MS-COCO and MM-CelebA-HQ sample images generated by DP-Wavelet. \textit{Left}. Original image, \textit{Center}. Final generated image, \textit{Right}. Low resolution intermediary. \label{fig:more_dp_wavelet}}
    \Description{Low resolution DP-Wavelet samples}
\end{figure}

\section{Privacy guarantee of DP-Wavelet}
\label{app:dp_proof}

\begin{theorem}
    The output of DP-Wavelet (Algorithm~\ref{alg:dp_wavelet}) is $(\varepsilon, \delta)$-DP.
\end{theorem}

\begin{proof}
    Let $\mathcal{A}_{\text{DP}}$ denote the $(\varepsilon,\delta)$-DP optimization algorithm (e.g., DP-SGD or DP-Adam) used to finetune the model $\mathcal{T}$ in Algorithm~\ref{alg:dp_wavelet} into a $(\varepsilon,\delta)$-DP model $\hat{\mathcal{T}}$. Using the fact that $\{x_i\}$ is public data, the fact that $\hat{\mathcal{T}}$ is $(\varepsilon,\delta)$-DP, and Proposition~\ref{prop:ppp}, it holds that the output $\{\hat{t}_i\}$ in the second stage of DP-Wavelet is also $(\varepsilon,\delta)$-DP. Similarly, using the fact that $\mathcal T$ and $\mathcal D$ are public and the same Proposition~\ref{prop:ppp}, it similarly follows that $\{s_i\}$ and the final output $\{\hat{y}_i\}$ are  $(\varepsilon,\delta)$-DP.
\end{proof}

\section{Discussion of training speed}
\label{app:train_speed}

DP-Wavelet is significantly faster to train compared to DP-LDM (\textbf{42x faster}) and DP-LlamaGen (\textbf{212x faster}). DP-LDM also requires more 2x resources compared to DP-Wavelet, despite having a similar number of parameters. We give some hypotheses on the faster training speed and fewer resources required by DP-Wavelet compared to DP-LDM and DP-LlamaGen:
\begin{enumerate}
    \item DP-LlamaGen has more than 2x the number of model parameters as the other methods.
    \item The most expensive (runtime and storage-wise) computations during DP training in DP-LDM are unrolling convolution layers during backpropation in the UNet model. In contrast, DP-Wavelet's most expensive computations are large matrix multiplications, which are significantly less costly. This difference is especially costly during DP training when example-level gradients must be materialized, i.e., one gradient for every example in a batch.
    \item Both DP-LDM and DP-Wavelet are implemented in JAX and are efficiently compiled to XLA to be optimized on TPUs. In particular, this enables operator fusion, which can significantly improve costly linear algebra operations. In contrast, DP-LlamaGen (and the LlamaGen backbone) are primarily built around the PyTorch framework and only natively support GPU acceleration.
\end{enumerate}

\section{More motivating examples}

\label{app:synth_motivation}
Another motivating setting is when the text for each image is private, but there exists a DP model $C$ that accurately generates the overall distribution of text. Our proposed DP-Wavelet method can be modified to handle this setting by (a) still DP finetuning on the private image-text pairs in the first stage, but (b) using the DP texts generated by $C$ to sample the coarse wavelet tokens in the second stage. 

In practice, there will be degradation in the distribution between the DP texts and the private texts due to quality gap of $C$, so our work focuses on the quality of images as if we had access to the private texts (to isolate the quality-privacy trade-off).

\end{document}